\documentclass{ecai} 

\usepackage{latexsym}
\usepackage{amssymb}
\usepackage{amsmath}
\usepackage{amsthm}
\usepackage{booktabs}
\usepackage{enumitem}
\usepackage{graphicx}
\usepackage{color}

\usepackage{mathtools}
\usepackage{algorithm}
\usepackage[noend]{algpseudocode}
\usepackage{listings}
\usepackage{tikz}
\usetikzlibrary{arrows,shapes,automata,petri,positioning,calc}
\usepackage{bm}
\newcommand{\matr}[1]{\mathbf{#1}}  

\newtheorem{theorem}{Theorem}

\newtheorem{proposition}[theorem]{Proposition}

\newtheorem{definition}{Definition}
\newtheorem{example}{Example}

\newcommand{\BibTeX}{B\kern-.05em{\sc i\kern-.025em b}\kern-.08em\TeX}

\newcommand*{\inlineimg}[1]{%
    \raisebox{-0.2\baselineskip}{%
        \includegraphics[
        height=0.9\baselineskip,
        width=0.9\baselineskip,
        keepaspectratio,
        ]{#1}%
    }%
}

\begin{document}


\begin{frontmatter}


\paperid{1593} 


\title{Differentiable Logic Programming for Distant Supervision}


\author[A]{\fnms{Akihiro}~\snm{Takemura}\orcid{0000-0003-4130-8311}\thanks{Corresponding Author. Email: atakemura@nii.ac.jp}}
\author[A]{\fnms{Katsumi}~\snm{Inoue}\orcid{0000-0002-2717-9122}}

\address[A]{National Institute of Informatics, Tokyo, Japan}


\begin{abstract}
We introduce a new method for integrating neural networks with logic programming in Neural-Symbolic AI (NeSy), aimed at learning with distant supervision, in which direct labels are unavailable. 
Unlike prior methods, our approach does not depend on symbolic solvers for reasoning about missing labels. 
Instead, it evaluates logical implications and constraints in a differentiable manner by embedding both neural network outputs and logic programs into matrices. 
This method facilitates more efficient learning under distant supervision. 
We evaluated our approach against existing methods while maintaining a constant volume of training data.
The findings indicate that our method not only matches or exceeds the accuracy of other methods across various tasks but also speeds up the learning process. 
These results highlight the potential of our approach to enhance both accuracy and learning efficiency in NeSy applications.
\end{abstract}

\end{frontmatter}


\section{Introduction}

Neural-Symbolic AI (NeSy) \cite{hitzlerNeuroSymbolicArtificialIntelligence2022,hitzlerCompendiumNeurosymbolicArtificial2023} is a field of research aimed at combining neural networks with symbolic reasoning.
While deep learning is capable of learning complex representations from input-output pairs, it requires a large amount of training data and struggles with tasks that require logical reasoning.
On the other hand, learning with symbolic reasoning can be done with small amounts of data, but it is sensitive to noise and unable to handle non-symbolic data.
In NeSy, it is crucial to combine the roles of neural networks and symbolic reasoning in a way that leverages their respective strengths.

There are various methods for implementing NeSy, including associating the continuous-valued parameters of neural networks (NN) with logical language and using the results of logical reasoning as the value of the loss function (e.g., semantic loss \cite{DBLP:conf/icml/XuZFLB18}).
Also known are methods that combine symbolic solvers with neural networks, e.g., DeepProbLog \cite{manhaeveDeepProbLogNeuralProbabilistic2018} uses Problog for probabilistic logic programming, and NeurASP \cite{yangNeurASPEmbracingNeural2020} uses clingo for answer set programming (ASP). 
These methods that internally call solvers often encapsulate computationally expensive problems such as weighted model counting or enumerating stable models in each iteration during learning.

Alternative methods have been proposed that embed inference traditionally done by symbolic reasoning solvers into vector spaces and perform symbolic reasoning using linear algebra \cite{sakamaLogicProgrammingTensor2021}. 
One such method embeds logic programs into vector spaces and designs appropriate loss functions based on the semantics of non-monotonic reasoning to compute the results of reasoning in a differentiable manner \cite{aspisStableSupportedSemantics2020,takemuraGradientBasedSupportedModel2022}. 
However, these methods have issues such as not being able to directly handle logical constraints and not being applicable to neural network learning as is. 
Thus, in this paper, we propose a method that enables learning in neural networks for NeSy tasks using logical programs that include constraints.

Distant supervision is a method of generating labeled data for learning using rules, external data, or knowledge bases and was proposed by \citet{mintzDistantSupervisionRelation2009} as a method to train classifiers for relation extraction based on information from knowledge bases. 
In NeSy, tasks where label information is provided through symbolic reasoning are commonly used, with MNIST Addition 
 \cite{manhaeveDeepProbLogNeuralProbabilistic2018} being a representative task. 
In this task, pairs of handwritten digits are input, and the goal is to learn the classification of handwritten digits with the sum of the digits provided as the label (e.g., \inlineimg{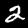}\(+\)\inlineimg{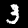}\(=5\)). 
Unlike the usual MNIST classification, in MNIST Addition, the labels are not given for each image individually. 
In this case, the relationship between the sum given as the label and the digits corresponding to the images is expected to be provided through symbolic reasoning.

In this paper, we propose a novel architecture for NeSy systems \cite{manhaeveDeepProbLogNeuralProbabilistic2018,yangNeurASPEmbracingNeural2020} that integrates differentiable logic programming \cite{aspisStableSupportedSemantics2020,takemuraGradientBasedSupportedModel2022} and neural networks.
This paper makes the following contributions:
\begin{enumerate}
    \item We propose a novel architecture that integrates neural networks with logic programming through a differentiable approach. 
    This method facilitates the direct evaluation of logical implications and constraints using differentiable operations, thus enabling effective learning under distant supervision without relying on symbolic solvers for reasoning about missing labels. 
    \item We demonstrate through experiments with a constant volume of training data that our proposed method not only matches but, in some cases, exceeds the accuracy of existing approaches that utilize symbolic solvers. 
    Moreover, we achieved a significant reduction in the training time for neural networks, highlighting substantial gains in computational efficiency.
\end{enumerate}

The structure of this paper is as follows.
After the preliminaries in Section \ref{sec:preliminaries}, Section \ref{sec:semantics} introduces the logic programming semantics in vector spaces.
Section \ref{sec:learning} presents our proposed method for using differentiable logic programming for distant supervision.
Section \ref{sec:experiments} presents the results of experiments and comparison to the state of the art NeSy methods.
Section \ref{sec:relatedwork} covers the related works in the literature.
Finally, Section \ref{sec:conclusion} presents the conclusion.


\section{Preliminaries}\label{sec:preliminaries}

A \textit{normal logic program} \(P\) is a set of \textit{rules} of the form:
\begin{equation} \label{eq:defnlp}
    A \leftarrow A_1 \wedge \dots \wedge A_m \wedge \neg A_{m+1} \wedge \dots \wedge \neg A_n\
\end{equation}
where \(A\) and \(A_i (n \geq m \geq 0)\) are atoms.
In this paper, the terms `normal logic program', `logic program', and `program' are used interchangeably.
An \textit{atom} is a predicate with some arity, e.g., \(p(X,Y)\), where variables are represented by upper case characters, and predicates and constants are represented by lower case characters.
A literal is either an atom \(p\), or its negation \(\neg p\).
The atom \(A\) in (\ref{eq:defnlp}) is the \textit{head} and \(\{A_1, \dots, A_n\}\) is the \textit{body} of a rule.
For each rule \(R_i\) of the form (\ref{eq:defnlp}), define \(head(R_i)=A\), \(body^{+}(R_i)=\{ A_1, \dots, A_m \}\) and \(body^{-}(R_i)=\{ A_{m+1}, \dots, A_n \}\).

The \textit{Herbrand universe} of a logic program \(P\) is the set of all ground terms in the language of \(P\), i.e., terms composed of function symbols and constants that appear in \(P\).
The \textit{Herbrand base} \(B_P\) is the set of atoms that can be formed from the relations of the program and terms in the Herbrand universe.
We assume that the Herbrand base \(B_P\) of a program to be lexicographically ordered.

A rule with an empty body is a \textit{fact}.
A program \(P\) is \textit{definite} if no rule in \(P\) contains negation as failure.
A program, a rule, or an atom is ground if it is variable free.
A program \(P\) is semantically identified with its ground instantiation, \(ground(P)\), by substituting variables in \(P\) by elements of its Herbrand universe in every possible way.

An \textit{interpretation} \(I \subseteq B_P\) satisfies a rule \(R_i\) of the form (\ref{eq:defnlp}) if \(body^{+}(R_i) \subseteq I\) and \(body^{-}(R_i) \cap I = \emptyset\) imply \(A \in I\).
An interpretation that satisfies every rule in a program \(P\) is a \textit{model} of the program.
A model of a program is \textit{supported} if for each atom \(p \in I\), there exists a ground rule such that \(I\) satisfies its body \cite{aptTheoryDeclarativeKnowledge1988}.
A model \(M\) is minimal if there is no model \(J\) of \(P\) such that \(J \subset M\).
A definite program has a unique minimal model, which is the least model.

Given a normal logic program \(P\) and an interpretation \(I\), the \textit{reduct} \(P^I\), which is a ground definite program, is constructed as follows: a ground rule \(A \leftarrow L_1,\dots,L_m\) is in \(P^I\) iff there is a ground rule of the form (\ref{eq:defnlp}) such that \(body^{-}(R_i) \cap I = \emptyset\).
If the least model of \(P^I\) is identical to \(I\), then \(I\) is a \textit{stable model} of \(P\) \cite{gelfondStableModelSemantics1988}.
For a definite program, the stable model coincides with the least model.
A stable model is always supported, but the converse does not hold in general.

Supported models can be computed as the models of \textit{Clark's completion} \cite{clarkNegationFailure1978}.
Let \(heads(P, a)\) be the set of rules in \(P\) whose head is \(a\). 
The completion of \(P\), denoted \(comp(P)\), is the set of clauses 
\begin{equation}\label{eq:defcomp}
    a \leftrightarrow \bigvee_{R_i \in heads(P, a)} body(R_i)
\end{equation}
for all \(a \in B_P\). 
A model of \(comp(P)\) is a supported model of \(P\) \cite{aptTheoryDeclarativeKnowledge1988}.

Let \(I \subseteq B_P\) be an interpretation of \(P\). 
The relation \(\models\) is defined as follows: for a rule \(R_i\) of the form (\ref{eq:defnlp}), \(I\) \textit{satisfies} \(R_i\) if \(head(R_i) \cap I \neq \emptyset\) whenever \(body(R_i) \subseteq I\), and denoted as \(I \models R_i\); for a program \(P\), \(I\) satisfies \(P\) if \(I \models R_i\) for all \(R_i \in P\); for a formula \(F = F_1 \vee \dots \vee F_k \ (k \geq 0)\), \(I \models F\) iff there is a \(F_i\ (k \geq i \geq 1)\) such that \(I \models F_i\), i.e., the empty disjunction is false.
Let \(comp(R_p)\) denote the completed rule (\(p \leftrightarrow body(R_{p1}) \vee \dots \vee body(R_{pj})\)) for the atom \(p\), then \(p \in I\) iff \(I \models comp(R_p)\).


\section{Semantics}\label{sec:semantics}

In this section, we consider the semantics of ground normal logic programs in vector spaces.
First, we introduce the necessary notations. 
Matrices are denoted using bold uppercase letters (\(\mathbf{M}\)), and vectors are denoted using bold lowercase letters (\(\mathbf{v}\)). 
The element in the \(i\)-th row and \(j\)-th column of a matrix is denoted by \(\mathbf{M}_{ij}\), and the \(i\)-th element of a vector is denoted by \(\mathbf{v}_i\).
The slice of the \(i\)-th row of a matrix is denoted by \(\mathbf{M}_{i:}\), and the slice of the \(j\)-th column is denoted by \(\mathbf{M}_{:j}\).
Variables are denoted by upper case letters, and constants and predicates are denoted by lower case letters; e.g., in \(sum(L)\), \(L\) is a variable and \(sum/1\) is a predicate with arity 1.

\subsection{Embedding Normal Logic Programs}

Given a ground normal logic program \(P\), we introduce two matrices that jointly represent the program.
The program matrix represents the bodies of the rules in the program, and the head matrix represents their disjunctions.
This is an alternative formulation to the embedding approach described by \citet{sakamaLogicProgrammingTensor2021}.
\begin{definition}[Program Matrix]
    Let \(P\) be a ground normal logic program with \(R\) rules and the size of its Herbrand base be \(|B_P|=N\).
    Then \(P\) is represented by a binary matrix \(\matr{Q} \in \{0,1\}^{R \times 2N}\) such that \(i\)-th row corresponds to the body of the \(i\)-th rule \(R_i\): \(\matr{Q}_{ij}=1\) if \(a_j \in body^{+}(R_i)\), \(\matr{Q}_{i(N+j)}=1\) if \(a_j \in body^{-}(R_i)\), and \(\matr{Q}_{ij}=0\) otherwise.
\end{definition}
\begin{definition}[Head Matrix]
    Let \(\mathbf{D} \in \{0,1\}^{(N \times R)}\) be the head matrix associated with \(P\).
    Then the element \(\mathbf{D}_{ji} = 1\) if the head of rule \(R_i (1 \leq i \leq R)\) is \(a_j (1 \leq j \leq N)\), and \(0\) otherwise.
\end{definition}

\begin{example}\label{ex:program_1}
    Consider the following program \(P_1\) with 3 rules:
    \begin{equation}
        (R_1)\  a \leftarrow c \wedge \neg b \qquad (R_2)\  a \leftarrow a \qquad (R_3)\  b \leftarrow \neg a
    \end{equation}
    \(P_1\) is encoded into a pair of matrices \(\left(\matr{Q}, \matr{D}\right)\):
    \begin{equation}
        \let\quad\enspace
        \matr{Q}  = \bordermatrix{ & a & b & c & \neg a & \neg b & \neg c \cr
                               R_1 & 0 & 0 & 1 & 0      & 1      & 0      \cr
                               R_2 & 1 & 0 & 0 & 0      & 0      & 0      \cr
                               R_3 & 0 & 0 & 0 & 1      & 0      & 0      
        } \;
        \matr{D}  = \bordermatrix { & R_1 & R_2 & R_3 \cr
                                  a & 1   & 1   & 0   \cr
                                  b & 0   & 0   & 1   \cr
                                  c & 0   & 0   & 0   
        }
    \end{equation}
\end{example}

\(\matr{Q}\) represents the bodies of the rules, which are the conjunctions of the literals appearing in the bodies.
For example, \(\matr{Q}_{1:}\) represents the body of \(R_1\), \(\left(c \wedge \neg b\right)\).
\(\matr{D}\) represents the disjunctions of the bodies of the rules sharing the same head.
For example, \(\matr{D}_{1:}\) represents the disjunction \(body(R_1) \vee body(R_2) = \left(c \wedge \neg b\right) \vee a\).
Together, \(\matr{Q}\) and \(\matr{D}\) represent the logic program \(P\).

\subsection{Evaluating Embedded Normal Logic Programs}

We consider the conjunction appearing in the bodies of the rules as the negation of disjunctions of negated literals using De Morgan's law, i.e., \(L_1 \wedge \dots \wedge L_n = \neg \left( \neg L_1 \vee \dots \vee \neg L_n \right)\).
This means that when evaluating the body of a rule, instead of checking whether all literals hold (as in \cite{takemuraGradientBasedSupportedModel2022}), we can count the number of false literals and check whether the count exceeds 1.
To this end, we introduce a piecewise linear function \(\mathrm{min}_1(x) = \mathrm{min}(x,1)= \mathrm{ReLU}(1-x)\), which gives 1 for \(x \geq 1\).
This function is almost everywhere differentiable (except at \(x=1\)), which allows gradient-based optimization to be applied effectively.

To evaluate normal logic programs in vector spaces, we introduce the vectorized counterparts of interpretation and model.
\begin{definition}[Interpretation Vector]
    Let \(P\) be a ground normal logic program.
    An interpretation \(I \subseteq B_P\) is represented by a binary vector \(\matr{v} = (\matr{v}_1,\dots,\matr{v}_N)^\intercal \in \mathbb{Z}^{N}\) where each element \(\matr{v}_i\,(1\leq i \leq N)\) represents the truth value of the proposition \(a_i\) such that \(\matr{v}_i=1\) if \(a_i \in I\), otherwise \(\matr{v}_i=0\). 
    We assume propositional variables share the common index such that \(\matr{v}_i\) corresponds to \(a_i\), and we write \(\mathrm{idx}(\matr{v}_i) = a_i\).
\end{definition}
\begin{definition} [Complementary Interpretation Vector]
     The complementary interpretation vector \(\matr{w} \in \mathbb{Z}^{2N}\) is a binary vector, which is a concatenation of the interpretation vector \({\matr{v}}\) and its complement: \(\matr{w} = [\matr{v};\matr{1}_N-\matr{v}]\).
\end{definition}

\begin{proposition}(Embedding Models of Normal Logic Programs) \label{prop:embed_model}
    Let \(\matr{P} = (\matr{Q}, \matr{D})\) be an embedding of a ground normal logic program \(P\), \(dist(\matr{\cdot},\matr{\cdot})\) be a distance function in a metric space, \(\matr{v}\) be an interpretation vector representing \(I \subseteq B_P\), and \(\matr{w}\) be its complementary interpretation vector.
    Then, for an interpretation vector \(\matr{v}\),
    \begin{equation*}
        I \models comp(P) \ \mathrm{iff}\ dist\bigg(\matr{v}, \mathrm{min_1}\bigg(\mathbf{D}\Big(\mathbf{1}-\mathrm{min_1}\big(\mathbf{Q}(\mathbf{1}-\mathbf{w})\big)\Big)\bigg)\bigg) = 0
    \end{equation*}
\end{proposition}  
\begin{proof} (Sketch; full proof in the Appendix \cite{takemuraDifferentiableLogicProgramming2024}.)
    A row slice of the program matrix \(\matr{Q}_{i:}\) corresponds to the body of a rule \(R_i\), so the matrix-vector products \(\matr{Q}_{i:}\matr{w}\) and \(\matr{Q}_{i:}(\matr{1}-\matr{w})\) computes the number of true and false literals in \(I\), respectively.
    The conjunctions can be computed as the negation of disjunctions of negated literals using De Morgan's law, i.e., \(\matr{1}-\mathrm{min_1}(\mathbf{Q}_{i:}(\mathbf{1}-\mathbf{w}))\).
    
    Let \(\{R^{a_i}\} = \{a_i \leftarrow B_j,\dots,a_i \leftarrow B_t\}\) be the set of rules that share the same head atom \(a_i\), where \(B_j\) denote the rule bodies, and \(B_j \vee \dots \vee B_t\) be the disjunction of the rule bodies.
    By construction of the head matrix \(\matr{D}\), \(\matr{D}_{i:}(\matr{1}-\mathrm{min_1}(\mathbf{Q}(\mathbf{1}-\mathbf{w})))\) computes the number of true rule bodies that share the same head.
    Thus, \(\matr{h}_i = \mathrm{min_1}\bigg(\mathbf{D}_{i:}\Big(\mathbf{1}-\mathrm{min_1}\big(\mathbf{Q}(\mathbf{1}-\mathbf{w})\big)\Big)\bigg) = 1\) if there is at least one rule body that is true in \(I\) and in the disjunction \(B_j \vee \dots \vee B_t\), and 0 otherwise.
    Then computing \(\matr{h}_i\) corresponds to the evaluation of \(I \models comp(R_{a_i})\).
    This can be generalized to the entire matrix.
    

    Let \(\matr{h}^I = \mathrm{min_1}\bigg(\mathbf{D}\Big(\mathbf{1}-\mathrm{min_1}\big(\mathbf{Q}(\mathbf{1}-\mathbf{w})\big)\Big)\bigg)\), then the second part of the iff relation is simplified to \(dist(\matr{v}, \matr{h}^I)) = 0\).
    \begin{itemize}
        \item If \(I \models comp(P)\), then \(dist(\matr{v}, \matr{h}^I) = 0\). \\
        Suppose \(I \models comp(R_{a_i})\), then there is at least one rule body that is true in \(I\), so \(\matr{h}_i^{I} = 1\).
        Otherwise, when we have \(I \not\models comp(R_{a_i})\), \(\matr{h}_i^{I} = 0\).
        Therefore, it holds that \(\matr{h}_i^{I} = \matr{v}_i\), and since the index \(i\) is arbitrary, we have \(\matr{v} = \matr{h}_i^{I}\), i.e., \(dist(\matr{v}, \matr{h}^I) = 0\).
        
        \item If \(dist(\matr{v}, \matr{h}^I) = 0\), then \(I \models comp(P)\). \\
        Consider \(dist(\matr{v}, \matr{h}^I) = 0\). 
        For \(\matr{h}^I_i = 1\), there is at least one rule body that is true in \(I\), and for \(\matr{h}^I_i = 0\), there is no rule body that is true in \(I\). 
        Since we have \(\matr{v}_i = \matr{h}^I_i\), for \(\matr{v}_i = \matr{h}^I_i = 1\), \(a_i \leftrightarrow \bigvee_{R_j \in heads(P, a_i)} body(R_j)\) is satisfied and denote \(I \models comp(R_{a_i})\), and for \(\matr{v}_i = \matr{h}^I_i = 0\) denote \(I \not\models comp(R_{a_i})\).
        Since the index \(i\) is arbitrary, we conclude \(I \models comp(P)\).
    \end{itemize}
\end{proof}

\begin{example}\label{ex:program1_implication} (Example \ref{ex:program_1} contd.)
    Consider the program \(P_1\) and corresponding matrices \(\matr{Q}\) and \(\matr{D}\) from Example \ref{ex:program1_implication}.
    This program has 2 supported models \(\{ \{a\}, \{b\} \}\).
    Let \(\matr{v}^{\{a\}} = (1 \; 0 \; 0 )^{\top}\) represent the interpretation \(\{a\}\), and \(\matr{w}\) be its complementary interpretation vector.
    We have \(\matr{h}^I = (1 \; 0 \; 0 )^{\top} = \matr{v}\), and take the Euclidean distance: \(dist(\matr{v},\matr{h}^I)=\sqrt{\sum_{i=1}^{3}(\matr{v}_i - \matr{h}^I_i)} = 0\).
    Therefore, according to Proposition \ref{prop:embed_model}, \(I \models comp(P_1) \).
\end{example}

The vector \(\matr{h} = \mathrm{min_1}\bigg(\mathbf{D}\Big(\mathbf{1}-\mathrm{min_1}\big(\mathbf{Q}(\mathbf{1}-\mathbf{w})\big)\Big)\bigg)\) serves as the \textit{head vector}, which is an indicator vector representing true atoms following the evaluation of rule bodies in the logic program.
This will be used later to define the loss function in Section \ref{sec:lossfunction}.

\subsection{Embedding and Evaluating Constraints}

A constraint is a rule with an empty head, e.g., \(\leftarrow a \wedge b\) represents a constraint where \(a\) and \(b\) must not both be true simultaneously.
Since constraints are rules in a program, we embed them into a constraint matrix \(\matr{C}\) in the same manner as the program matrix \(\matr{P}\).
Note that we do not require the head matrix because constraints have empty heads.
\begin{definition}[Constraint Matrix]
    Let \(C = \{C_1,\dots,C_k\}=\{\leftarrow B_1,\dots,\leftarrow B_k\}\) be the set of constraints in a program \(P\) with \(|B_P|=N\).
    Then the matrix corresponding to the constraints is \(\matr{C} \in \{0,1\}^{(k \times 2N)}\) such that \(i\)-th row corresponds to the body of the \(i\)-th constraint \(C_i\): \(\matr{C}_{ij}=1\) if \(a_j \in body^{+}(C_i)\), \(\matr{C}_{i(N+j)}=1\) if \(a_j \in body^{-}(C_i)\) and \(\matr{C}_{ij}=0\) otherwise.
\end{definition}

To evaluate the constraints, we check whether the bodies of the constraint rules are in \(I\): given a constraint \(R_i\), if \(body(R_i) \subseteq I\) then the constraint is violated; otherwise it is satisfied.
\begin{proposition}(Evaluating Constraints)
    Let \(\matr{C}\) be an embedding of constraints \(C\), \(dist(\cdot,\cdot)\) be a distance function in a metric space, \(\matr{v}\) be an interpretation vector representing \(I \subseteq B_P\), and \(\matr{w}\) be its complementary interpretation vector.
    Then, for an interpretation vector \(\matr{v}\), it holds that 
    \(I \not\models C\) iff \(dist(\matr{1}, \mathrm{min}_1(\matr{C}(\matr{1}-\matr{w})))=0\).
\end{proposition}
\begin{proof}
    Proved similarly to Proposition 1.
    Let \(\matr{c}^I = \mathrm{min_1}(\mathbf{C}(\mathbf{1}-\mathbf{w}))\).
    Consider the \(i\)-th constraint \(C_i\) and the corresponding row slice \(\matr{C}_{i:}\).
    The existence of at least one false literal in the body is computed by \(\matr{c}_i^I = \mathrm{min_1}(\mathbf{C}_{i:}(\mathbf{1}-\mathbf{w}))\), where \(\matr{c}_i^I = 1\) if there is a false literal and \(\matr{c}_i^I = 0\) otherwise, i.e., when \(\matr{c}_i^I = 0\), the body is satisfied and the constraint is violated.

    Suppose \(I \not\models C_i\), then there is at least one false literal in the body of \(C_i\), so \(\mathrm{min}_1(\matr{C}_{i:}(\matr{1}-\matr{w})))=1\).
    Repeat this for all \(C_i \in C\), we obtain a 1-vector, which means there is at least one false literal in the bodies of all constraints.
    By definition, \(dist(\matr{1},\matr{1})=0\). The converse can be proved similarly.
    
        
\end{proof}

\begin{example}\label{ex:constraint}
    Consider the constraint \(\leftarrow a \wedge b\). Then we have:
    \begin{equation}
        \let\quad\enspace
        \matr{C}  = \bordermatrix{ & a & b & c & \neg a & \neg b & \neg c \cr
                               C_1 & 1 & 1 & 0 &      0      & 0      & 0        
        } \;
    \end{equation}
    We examine two scenarios: in the first, the constraint is violated, and in the second, it is not.
    Let \(\matr{c}^I = \mathrm{min_1}(\mathbf{C}(\mathbf{1}-\mathbf{w}))\) and \(dist(\cdot,\cdot)\) denote the Euclidean distance, then, for the following cases,
    \begin{itemize}
        \item \(\matr{v}^{\{a,b\}} = (1\; 1\; 0)\): we obtain \(\matr{c}^{\{a,b\}} = (0)\), and \(dist(\matr{1},\matr{c}^{\{a,b\}})=1\), so we conclude that the constraint is violated.
        \item \(\matr{v}^{\{a\}}   = (1\; 0\; 0)\): we obtain \(\matr{c}^{\{a\}} = (1)\), and \(dist(\matr{1},\matr{c}^{\{a,b\}})=0\), so we conclude \(I \not\models C\).
    \end{itemize}
\end{example}

For later use in the loss function (Section \ref{sec:lossfunction}), we define the \textit{constraint violation vector} \(\matr{c}'\) as \(\matr{c}' = \matr{1} - \matr{c} = \matr{1} - \mathrm{min}_1(\matr{C}(\matr{1}-\matr{w}))\).
Intuitively, this modification turns \(\matr{c}\) into an indicator vector, where each element \(\matr{c}'_i = 1\) means that the \(i\)-th constraint is violated.

\section{Learning with Differentiable Logic Program}\label{sec:learning}

In this section, we show how the aforementioned differentiable logic program semantics can be used to train neural networks.
Although our method supports both implication and constraint rules, it is not always necessary to use both of them for learning, as we shall show later in the experimental section.
Specifically, for the NeSy tasks we studied, using exclusively either one of implication or constraint rules is enough to achieve competitive accuracy.
On the other hand, in NeurASP \cite{yangNeurASPEmbracingNeural2020} for example, the observation atoms are typically given as integrity constraints in ASP rules to compute stable models, and implication rules are defined similarly to ours.
Consequently, we included a combination of both implication rules and constrains in our experiments to provide a thorough evaluation.

\subsection{Example: MNIST Addition}

The MNIST digit addition problem \cite{manhaeveDeepProbLogNeuralProbabilistic2018} is a simple distant supervision task that is commonly used in neural-symbolic literature.
The input consists of two images, and the output is their sum (e.g., \inlineimg{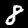}, \inlineimg{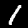}, \(9\)). 
The goal is to learn digit classification from the sum of digits, rather than from images with individually labeled digits, as in the usual MNIST digit classification.
For brevity, we shall focus on the single digit additions in this section.

We first introduce \textit{neural predicates} \cite{manhaeveDeepProbLogNeuralProbabilistic2018}, which act as an interface between the neural and logic programming parts.
More concretely, a neural predicate is a predicate which can take references to tensorized objects as arguments.
A \textit{neural atom} is an instance of a neural predicate, representing a specific combination of values or variables.
For example, in MNSIT Addition, we define two neural predicates, \(obs/4\) and \(label/3\).
\(obs(i_1,D_1,i_2,D_2)\) represents a situation where images \(i_1\) and \(i_2\) were classified as digits \(D_1\) and \(D_2\) ranging from 0 to 9, respectively.
\(label(i_1,i_2,S)\) represents the two images and their sum \(S\) ranging from 0 to 18.
Thus, we obtain 100 \(obs\) and 19 \(label\) neural atoms.

\subsubsection{Implication rules}

In general, we expect the label atoms to appear in the heads of the implication rules so that we can compare the output of the logic programming component with the label using a loss function such as binary cross entropy.
For a single digit MNIST Addition, the label is the integer values between 0 and 18 represented by the \(label/3\) predicate.
As for the individual rules, we enumerate the possible combinations of digits that sum to the given label, e.g.,
\begin{align}\label{pgm:mnist_implication}
    label(i_1,i_2,0) &\leftarrow obs(i_1,0,i_2,0). \nonumber \\ 
    label(i_1,i_2,1) &\leftarrow obs(i_1,0,i_2,1). \nonumber \\
    label(i_1,i_2,1) &\leftarrow obs(i_1,1,i_2,0). \nonumber \\
    \dots \nonumber \\
    label(i_1,i_2,18) & \leftarrow obs(i_1,9,i_2,9).
\end{align}
In this way, 100 rules with \(label/3\) in the heads can be instantiated, which results in the embedded program \(\matr{P}_{impl.}=(\matr{Q}_{impl.},\matr{D}_{impl.})\) where \(\matr{Q}_{impl.} \in \{0,1\}^{(100 \times 200)}\) and \(\matr{D}_{impl.} \in \{0,1\}^{(19 \times 100)}\).

\subsubsection{Constraints}

In the case of MNIST Addition, constraints can be represented with smaller number of rules compared to the implication rules, e.g.,
\begin{align}\label{pgm:mnist_constraint}
    & \leftarrow label(i_1,i_2,0) \wedge \neg obs(i_1,0,i_2,0). \nonumber \\
    & \leftarrow label(i_1,i_2,1) \wedge \neg obs(i_1,0,i_2,1) \wedge \neg obs(i_1,1,i_2,0). \nonumber \\
    & \dots \nonumber \\
    & \leftarrow label(i_1,i_2,18) \wedge \neg obs(i_1,9,i_2,9).
\end{align}
Intuitively, one can read the first rule as ``when the label is 0, both of the digits must be 0''.
This essentially amounts to enumerating all possible combinations of digits that sum to the label and adding them as negative literals to the rule bodies.
Preparing constraints for each label results in a constraint matrix \(\matr{C} \in \{0,1\}^{(19 \times 238)}\).

\subsubsection{Handling Neural Network Outputs}

Here, the outputs of the neural network combined with facts evident from the problem and labels are treated as a continuous interpretation.
Facts evident from the problem refer to, for example, the three digits (\(d_1,d_2,d_3\)) in the Apply2x2 task or the given digit number in the Member task.
The details of the Apply2x2 and Member tasks will be explained later in the experiment section.
The label information is also incorporated into this continuous-valued interpretation vector. 
\begin{definition}[Continuous-valued Interpretation Vector]
    \normalfont{}
    Let \(\matr{x} \in [0,1]^N\) be the output passed through the last layer of the neural network. 
    Let \(\matr{f} \in \{0,1\}^N\) represent facts from the problem that are not dependent on NN's output, and \(\matr{lb} \in \{0,1\}^N\) represent the label information available from the instance.
    If the number of elements in \(\matr{x}, \matr{f}\) or \(\matr{lb}\) is less than \(N\), pad appropriately with zeros. 
    The continuous-valued interpretation vector \(\matr{z}\) is computed as follows: \( \matr{z} = \matr{x} + \matr{f} + \matr{lb} \).
\end{definition}

In MNIST Addition, facts are not available from the problem settings, so we only focus on the neural network outputs and labels.
In this task, the inputs to the (convolutional) neural network are two images \((i_1,i_2)\).
After passing through the Softmax activation, we obtain two output vectors \(\matr{x}_1, \matr{x}_2 \in [0,1]^{10}\) as (probabilistic) outputs.
To map these vectors to 100 \(obs\) neural atoms, we compute their Cartesian product and obtain \(\matr{x} \in [0,1]^{100}\): \(\matr{x} = \matr{x}_1 \bigtimes \matr{x}_2 \).
Depending on the problem, the dimension of the continuous-valued interpretation vector may be different for implication and constraints.
In MNIST Addition, for implication rules it is not necessary to have \(label\) in the interpretation vector, as they are the heads of the rules.
On the other hand, it is necessary to include \(label\) in the interpretation vector for evaluating constraints, as they are present in the rule bodies.
It is imperative that the indexing remain consistent across all vectors and matrices; otherwise we risk neural network outputs being mapped to incorrect neural atoms.

The learning pipeline is shown in Figure \ref{fig:pipeline}.
Firstly, the input images are classified using the CNN, and we associate its probabilistic output with neural atoms \(obs/4\).
Then, using the neural atoms and label information, we obtain the logical loss which can be used to train the CNN with gradient backpropagation.

\begin{figure*}[ht]
    \centering
    \includegraphics[height=15em]{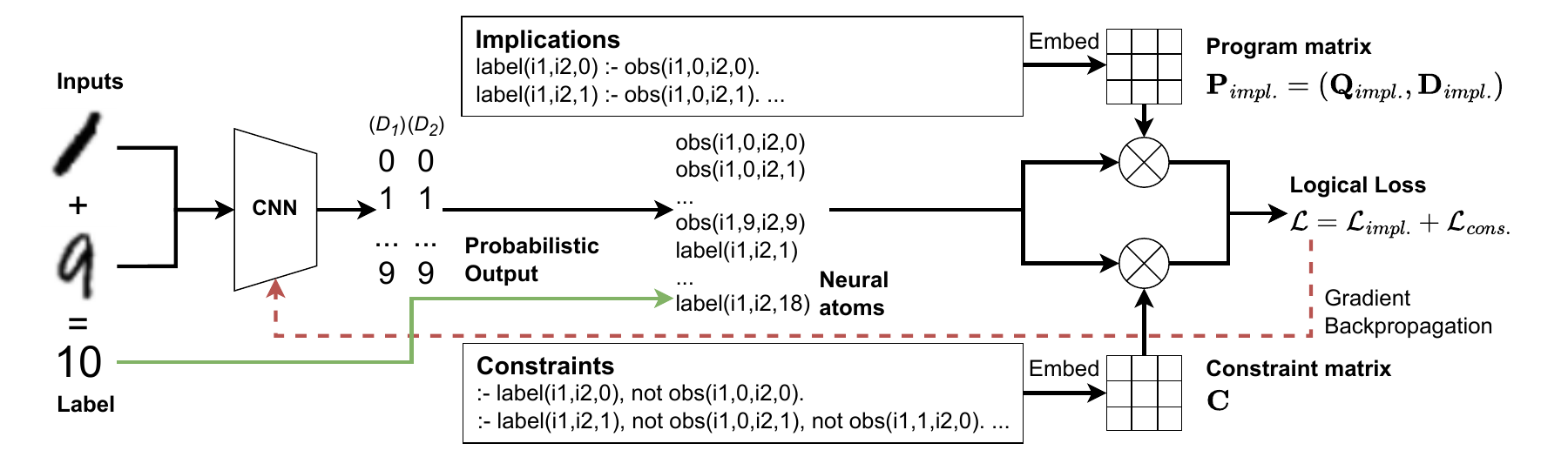}
    \caption{The learning pipeline for MNIST Addition.}
    \label{fig:pipeline}
\end{figure*}

\subsubsection{Loss Function}\label{sec:lossfunction}

Using the embedded program, continuous-valued interpretation vector and label information, the loss function is defined as follows:
\begin{align}
    \matr{h} &= \mathrm{min_1}\bigg(\mathbf{D}\Big(\mathbf{1}-\mathrm{min_1}\big(\mathbf{Q}(\mathbf{1}-\mathbf{w_z})\big)\Big)\bigg)\\
    \matr{c}' &= \matr{1} - \mathrm{min}_1(\matr{C}(\matr{1}-\matr{w_z})) \\
    \mathcal{L} & = \mathcal{L}_{impl.} + \mathcal{L}_{cons.} = BCE(\matr{h}, \matr{lb}) + BCE(\matr{c}', \matr{0})
\end{align}
where \(BCE\) stands for binary cross-entropy, \(\matr{lb}\) corresponds to the label vector, and \(\matr{0}\) is a zero vector with the same dimension as \(\matr{c}'\).
Note that \(\matr{w_z}\) here is the complementary interpretation vector based on the continuous-valued interpretation vector \(\matr{z}\) which contains neural atoms, facts and labels: \(\matr{w_z} = [\matr{z};\matr{1}_N-\matr{z}]\).
When the continuous valued interpretation vector is a binary one, \(\matr{h}\) corresponds to the interpretation \(I\).
If all atoms that are supposed to be true in \(I\) are true, then the BCE loss will be 0.
\(\matr{c}'\) corresponds to an indicator vector for the violated constraints. 
Thus, \(\matr{c}'\) should be all 0 when all constraints are satisfied.
Combining the aforementioned BCE's, the loss function will be 0 iff all implied neural atoms are true, and all constraints are satisfied.



The difference between programs (\ref{pgm:mnist_implication}) and (\ref{pgm:mnist_constraint}) lies in their structure: implication rules may contain label neural atoms in the head, whereas constraints always contain label neural atoms in the body.
Implication rules present a more straightforward approach for representing partial information in the form of logical rules, especially in scenarios where employing intermediate predicates is necessary or enumerating constraints is time-consuming.
In the context of the NeSy tasks examined in this study, we observed a consistent pattern where implication rules are typically comprised of straightforward forward inference rules, and constraints are formed from a fewer number of rules.
Each of these constraints contains a label neural atom alongside a series of negated observation atoms in their body.

Based on the definition of the combined loss function, it is clear that only one is necessary for accomplishing the MNIST Addition task.
The first BCE is essentially the same as the one for a 19-label multiclass classification task (labels spanning from 0 to 18), while the second BCE corresponds to a multiclass classification with an all-0 label.

The evaluation of implication rules ensures that if the correct label's corresponding atom is derived as the head, \(\mathcal{L}_{impl.}\) becomes 0. 
Similarly, if all constraints are satisfied, then \(\mathcal{L}_{cons.}\) becomes 0. 
Since both the evaluation of implication rules and constraints are defined in a (almost-everywhere) differentiable manner, it is possible to train the neural network using this loss function through backpropagation.

\section{Experiments}\label{sec:experiments}

\subsection{Task Description}

We studied the learning performance on the following NeSy tasks.

\subsubsection*{MNIST Addition \normalfont{\cite{manhaeveDeepProbLogNeuralProbabilistic2018}}}
The input consists of two MNIST images of digits (\(i_1,i_2\)), and the output is their sum (e.g., \inlineimg{images/8_31.png}, \inlineimg{images/1_3.png}, \(9\)). 
The goal is to learn image classification from the sum of digits, rather than from images with individually labeled digits, as in the usual MNIST classification. 
This paper deals with two types: single-digit additions (two images), and two-digit additions (four images).

\subsubsection*{ADD2x2 \normalfont{\cite{gauntDifferentiableProgramsNeural2017}}}
The input is four MNIST images of digits (\(i_{11},i_{12},i_{21},i_{22}\)) arranged in a 2x2 grid, and the output is four sums (\(s_1,s_2,s_3,s_4\)) calculated from each row and column of the grid (e.g., \inlineimg{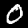}, \inlineimg{images/1_3.png}, \inlineimg{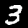}, \inlineimg{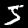}, \(1,8,3,6\)). 
The goal is to learn the classification problem of MNIST images from the four sums provided as labels.

\subsubsection*{APPLY2x2 \normalfont{\cite{gauntDifferentiableProgramsNeural2017}}}
The input consists of three numbers (\(d_1, d_2, d_3\)) and four handwritten operator images (\(o_1,o_2,o_3,o_4\)) arranged in a 2x2 grid, with the output being the results (\(r_1,r_2,r_3,r_4\)) of operations performed along each row and column of the grid. The operators are one of \(\{+, -, \times\}\). For example, the result for the first row is calculated as \(r_1 = (d_1\;op_{11}\;d_2)\;op_{12}\;d_3\) (e.g., \(1,2,4\), \inlineimg{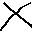}, \inlineimg{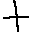}, \inlineimg{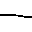}, \inlineimg{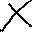}, \(6,-4,-2,12\)). The goal is to learn the classification of handwritten operators from the four results and three numbers given as labels.

\subsubsection*{MEMBER(n) \normalfont{\cite{tsamouraNeuralSymbolicIntegrationCompositional2021}}}
For n=3, the input consists of three images (\(i_1,i_2,i_3\)) and one number (\(d_1\)), with the output being a boolean indicating whether \(d_1\) is included in the three images (e.g., \inlineimg{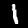}, \inlineimg{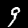}, \inlineimg{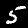}, \(4,0\)). For n=5, the input includes five images (\(i_1,...,i_5\)). The goal is to learn the classification problem of MNIST images from the numbers provided as labels. This paper deals with two types: n=3 and n=5.

\subsection{Implementation and Experimental Setup}

The methods introduced in the previous section were implemented in PyTorch. 
The convolutional neural network (CNN) used in the experiments is the same as those in \cite{manhaeveDeepProbLogNeuralProbabilistic2018} and \cite{yangNeurASPEmbracingNeural2020}. 
The experimental environment is an AMD Ryzen 7950X (16c/32t), 128GB RAM, and an NVIDIA A4000 (16GB), with settings to utilize the GPU as much as possible. 
The number of training data was 30,000 and 15,000 for Addition 1 and 2, respectively, and 10,000 for other tasks. 
Unless otherwise noted, the number of epochs and batch size for all tasks were set to 1, and Adam was used with a learning rate of 0.001.
Each experiment consisted of 5 repeated trials, and the average is reported.
The timeout was set to 30 minutes per trial.

\subsection{Results}

Table \ref{tab:matrix_size} shows the dimensions of the program matrix \(\matr{P}\), head matrix \(\matr{D}\), and constraint matrix \(\matr{C}\).
In terms of the number of rules in respective programs, the constraint matrix is usually smaller than the implication matrix, because each constraint rule (row) is often an enumeration of conditions for the target atom, whereas in the implication matrix, multiple rules (rows) can share the same head.
The rate of growth of the program matrix is highly task dependent; for example, adding two digits to the MNIST Addition task result in adding 10,000 elements to the matrix, whereas adding two digits to the set in the Member task results in adding a few thousand elements.
\begin{table}[htb]
    \centering
    \caption{Dimensions of program matrices}
    \begin{tabular}{l r r r}
    \hline
     & \(\matr{P}\) (implication) & \(\matr{D}\) (head) & \(\matr{C}\) (constraint) \\
    \hline
    Addition 1  & \(100\times200\)       & \(19\times100\)        & \(19\times238\)      \\
    Addition 2  & \(10000\times20000\)   & \(199\times10000\)    & \(199\times20398\)      \\
    Add2x2      & \(400\times800\)       & \(76\times400\)        & \(76\times952\)      \\
    Apply2x2    & \(11979\times2680\)    & \(10597\times11979\) & \(10597\times23874\)      \\
    Member 3    & \(40\times60\)         & \(20\times40\)         & \(40\times100\)      \\
    Member 5    & \(60\times100\)        & \(20\times60\)         & \(60\times140\)      \\
    \hline
    \end{tabular}
    \label{tab:matrix_size}
\end{table}

The accuracy and training time are reported in tables \ref{tab:accuracy} and \ref{tab:time}, respectively.
In the tables, I indicates training using only implication rules, C indicates training using only constraints, and I+C indicates training using both implication rules and constraints.
DPL, DSL and NASP are abbreviations for DeepProbLog \cite{manhaeveDeepProbLogNeuralProbabilistic2018}, DeepStochLog \cite{wintersDeepStochLogNeuralStochastic2022a} and NeurASP \cite{yangNeurASPEmbracingNeural2020}, respectively.

Table \ref{tab:accuracy} shows the average accuracy of MNIST digit classification and math operator classification.
It can be seen that while the proposed method achieved comparable accuracy to the comparison methods in MNIST Addition and Add2x2, it significantly outperformed the comparison methods in Apply2x2. 
However, in Member 3, the accuracy was lower. 
Comparing the use of implication rules, constraints, or both, there was no significant difference except for Member(n), where using either implication rules or constraints alone tended to result in higher accuracy than using both.
In particular, for Member 5 tasks, combining implication rules and constraints led to significantly worse performance compared to individual applications. 
This suggests that we might need to modify the combination loss or consider introducing a scheduler for improving training performance.

\begin{table}[htb]
    \centering
    \caption{Accuracy on digit and operator classification. The numbers in parentheses indicate timeouts (30 min).}
    \begin{tabular}{l r r r| rrr}
    \hline
    \multicolumn{1}{l}{Accuracy (\%)} & \multicolumn{3}{c}{Comparisons} & \multicolumn{3}{c}{Ours} \\
                & DPL   & DSL & NASP  & I & C & I+C   \\
    \hline
    Addition 1  & \textbf{97.8} & 95.8          & 97.7          & 97.7          & 97.5          & 97.4      \\
    Addition 2  & 97.7          & 97.8          & 97.8          & 97.5          & \textbf{97.9} & 97.8      \\
    Add2x2      & T/O(5)        & \textbf{98.0} & 97.5          & 97.7          & 97.6          & 97.9      \\
    Apply2x2    & 87.8          & 87.8          & 80.9          & \textbf{99.5} & 99.4          & 99.4      \\
    Member 3    & 92.3(3)       & \textbf{92.9} & 91.7          & 87.8          & 87.0          & 84.6      \\
    Member 5    & T/O(5)        & T/O(5)        & T/O(5)        & 86.3          & \textbf{86.4} & 69.5      \\
    \hline
    \end{tabular}
    \label{tab:accuracy}
\end{table}

Table \ref{tab:time} shows the average training times of each method.
Except for Apply2x2, it is evident that the proposed method can learn faster than existing methods.
Especially in Member 5, where comparison methods timed out, the proposed method processed 10,000 training instances in about 13 seconds, showing a significant speed difference.
The long training time for Apply2x2 can be attributed to the rather naive implementation used in this experiment.
In our implementation, Apply2x2 required a more fine-grained control over indexing, necessitating the use of less efficient 'for-loops' to avoid neural network outputs being mapped to incorrect neural atoms.
In contrast, in other tasks, the more efficient vectorized operations were used to compute the Cartesian product.

\begin{table}[htb]
    \centering
    \caption{Learning times of each method. The numbers in parentheses indicate timeouts (30 min).}
    \begin{tabular}{l r r r| rrr}
    \hline
    \multicolumn{1}{l}{Time (sec)} & \multicolumn{3}{c}{Comparisons} & \multicolumn{3}{c}{Ours} \\
                & DPL     & DSL & NASP   & I & C & I+C   \\
    \hline
    Addition 1  & 470.7     & \textbf{20.8}      & 84.7      & 31.4          & 31.4          & 35.9      \\
    Addition 2  & 1120      & 33.6      & 283.5     & 83.5          & \textbf{19.9}          & 87.4      \\
    Add2x2      & T/O(5)    & 35.8      & 131.1     & 16.9          & \textbf{16.2}          & 18.0      \\
    Apply2x2    & 323.3     & 127.7     & \textbf{25.5}      & 154.4         & 228.5         & 359.5      \\
    Member 3    & 1782(3)   & 398.4     & 191.9     & \textbf{10.3}          & \textbf{10.3}          & 11.9      \\
    Member 5    & T/O(5)    & T/O(5)    & T/O(5)    & \textbf{11.6}          & 11.7          & 13.4      \\
    \hline
    \end{tabular}
    \label{tab:time}
\end{table}

\begin{figure}[htb]
    \centering
    \includegraphics[width=\linewidth]{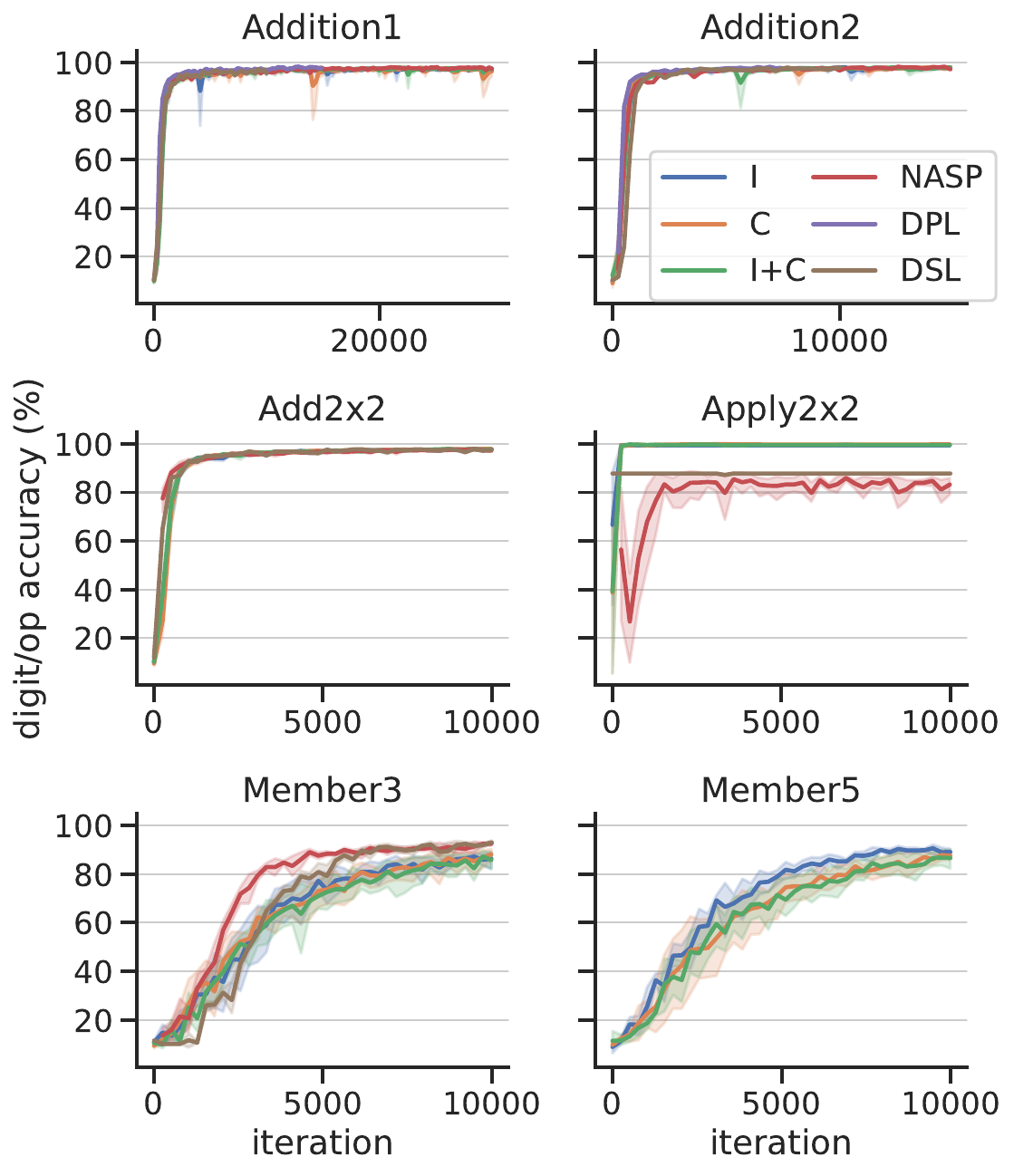}
    \caption{Test accuracy (\%) during training.}
    \label{fig:graph_training}
\end{figure}

Figure \ref{fig:graph_training} shows the test accuracy during training.\footnote{This is from a different set of experiments, as such, the results presented in this figure do not necessarily match those in Table \ref{tab:accuracy} and Table \ref{tab:time}. The timeout was set to 2 hours per trial.}
We see that the addition tasks (MNIST Addition and Add2x2) can be handled equally well by our method, NeurASP and DeepStochLog.
The difference is more pronounced for Apply2x2 and Member3, and in Apply2x2, we observe that DeepStochLog plateaus quickly while NeurASP fluctuates around 80\%.
In Member 3, NeurASP leads in terms of accuracy up to 5,000 iterations, other methods catch up after 9,000 iterations.
Finally, it is interesting that the implication rule only method performs well in Member 5 compared against constraint only or the combination of both.
This suggests that there might be tasks that can be learned better by certain types or combinations of programs, although it is difficult to know beforehand which one would perform the best.


\section{Related Work}\label{sec:relatedwork}

\citet{DBLP:conf/icml/XuZFLB18} introduced the semantic loss that leverages sentential decision diagrams for efficient loss calculation, enabling effective learning in both supervised and semi-supervised settings.
LTN \cite{badreddineLogicTensorNetworks2022} and LNN \cite{riegelLogicalNeuralNetworks2020a} embed first-order logic formulas using fuzzy logic within neural network architectures.
Our approach is similar to the semantic loss, in the sense that we evaluate the neural output using a differentiable loss function to train the neural network by backpropagation.
On the other hand, our approach does not require weighted model counting, nor direct embedding of logical operations, and the logic programming semantics is kept separate from the perception neural network.

Parallel to direct integration strategies, significant work has been conduced on coupling neural outputs with symbolic solvers.
For instance, DeepProbLog \cite{manhaeveDeepProbLogNeuralProbabilistic2018}, DeepStochLog \cite{wintersDeepStochLogNeuralStochastic2022a}, NeurASP \cite{yangNeurASPEmbracingNeural2020}, and NeuroLog \cite{tsamouraNeuralSymbolicIntegrationCompositional2021} facilitate inference by integrating the neural outputs as probabilities or weighted models within a symbolic solver.
While these approaches can utilize logical constraints and background knowledge represented by logical programs, the computational cost of symbolic reasoning can become a bottleneck during learning.
Distinct from the aforementioned coupling approaches, the focus shifts more towards reasoning rather than simultaneous learning in neural-symbolic systems proposed by \citet{eiterNeuroSymbolicASPPipeline2022a} and Embed2Sym \cite{aspisEmbed2SymScalableNeuroSymbolic2022}.
In the pipeline proposed by \citet{eiterNeuroSymbolicASPPipeline2022a}, the neural network is trained separately, and predictions that pass a predefined confidence threshold are then translated into logic programs for reasoning with ASP.

Various neural-symbolic approaches have been developed for symbolic rule learning, with varying degree of integration between logical reasoning and neural computation. Neural Theorem Provers \cite{rocktaschelEndtoendDifferentiableProving2017} and \(\partial\)ILP \cite{evansLearningExplanatoryRules2018} allow end-to-end differentiable learning of explanatory rules from data.
Similarly, frameworks like NeuroLog \cite{tsamouraNeuralSymbolicIntegrationCompositional2021}, NSIL \cite{cunningtonNeuroSymbolicLearningAnswer2023}, and the Apperception Engine \cite{evansMakingSenseRaw2021} integrate symbolic solvers to enhance the reasoning processes, utilizing the output from neural networks represented as logical constructs.
Additionally, \(\alpha\)ILP \cite{shindoAlphaILPThinking2023} and DFORL \cite{gaoDifferentiableFirstorderRule2024} extend the capabilities of inductive logic programming (ILP) by making the learning process differentiable.
In particular, \(\alpha\)ILP is designed for visual scene understanding, while DFORL focuses on relational data, demonstrating the versatility of differentiable ILP in handling diverse data types.
In Neural Logic Machines \cite{dongNeuralLogicMachines2019}, the architecture of the neural network itself is designed to mimic logical reasoning processes, thereby learning to approximate symbolic rules.
In contrast to these methods, which predominantly focus on enhancing the logic programming aspects within neural frameworks, our approach focuses more on enhancing the training of neural networks themselves, and we leave the inductive learning of logic programs for future work.

Recent advancements have explored various methods for representing logic programming semantics in vector spaces, each with distinct embedding strategies and computational techniques. 
Notably, \citet{sakamaLogicProgrammingTensor2021} proposed a method to embed logic programs into vector spaces, and compute logic programming semantics using linear algebra.
More specifically, their method allows for the computation of stable models as fixpoints through repeated tensor multiplications. 
However, a significant limitation from the standpoint of neural-symbolic integration is the non-differentiability of their method, which complicates direct integration with neural networks.
Other matrix-based approaches by \citet{sato3valuedSemanticsSupported2020} and \citet{takemuraGradientBasedSupportedModel2022} use binary program matrices to represent logic programs.
While \citet{sato3valuedSemanticsSupported2020}'s method is non-differentiable, the one proposed by \citet{takemuraGradientBasedSupportedModel2022} is based on a differentiable loss function, where the interpretation vector is treated as the input to this loss function, and the loss is minimized by updating the interpretation vector using the gradient of the loss function.
Similarly, \citet{aspisStableSupportedSemantics2020} proposed a method based on a root-finding algorithm, which presents yet another computational strategy in this domain.
In contrast to these existing methods, our approach computes supported models in vector spaces without imposing restrictive conditions on program structures that require the rewriting of same head rules, for example.
The method presented in this paper does not minimize the loss function by updating the interpretation vector using the gradient information, and the loss is reduced by updating the parameters of the neural network to predict the correct intermediate labels in a distant supervision setting.
To the best of our knowledge, no other implementation exists that utilizes the differentiable logic programming semantics for training neural networks in neural-symbolic settings.

\section{Conclusion}\label{sec:conclusion}

We proposed a method to assist the learning of neural networks with logic programs and verified its effectiveness in NeSy tasks. 
This method is based on a differentiable logic programming semantics, where continuous-valued interpretation vectors contain outputs of neural networks, and evaluation of implication rules and constraints are incorporated into a loss function, enabling the learning of neural networks under the distant supervision settings.
The experimental results showed that it is possible to achieve comparable accuracy to those based on symbolic solvers, and with an exception, the proposed method completed the neural network training much faster than the existing methods.

The findings of this study demonstrate the effectiveness of the approach based on the differentiable logic programming semantics for enabling high-accuracy and fast learning in NeSy. 
Future work includes applying the proposed method to more complex NeSy tasks.
Additionally, a more detailed analysis of the balance between constraints and implication rules is necessary. 
In the tasks addressed in this study, there was little difference between using constraints and implication rules, but this might vary by task.
Another limitation is that the semantics as presented currently is only valid for crisp 0-1 interpretations, so other continuous-valued interpretations do not necessarily have natural meanings associated with them.
To this end, it might be interesting to seek connections between the differentiable logic programming semantics and fuzzy logic.




\begin{ack}
This work has been supported by JST CREST JPMJCR22D3 and JSPS KAKENHI JP21H04905.
\end{ack}



\bibliography{ref}

\begin{thebibliography}{33}
\providecommand{\natexlab}[1]{#1}
\providecommand{\url}[1]{\texttt{#1}}
\expandafter\ifx\csname urlstyle\endcsname\relax
  \providecommand{\doi}[1]{doi: #1}\else
  \providecommand{\doi}{doi: \begingroup \urlstyle{rm}\Url}\fi

\bibitem[Apt et~al.(1988)Apt, Blair, and Walker]{aptTheoryDeclarativeKnowledge1988}
K.~R. Apt, H.~A. Blair, and A.~Walker.
\newblock Towards a theory of declarative knowledge.
\newblock In \emph{Foundations of Deductive Databases and Logic Programming}, pages 89--148. {Elsevier}, 1988.

\bibitem[Aspis et~al.(2020)Aspis, Broda, Russo, and Lobo]{aspisStableSupportedSemantics2020}
Y.~Aspis, K.~Broda, A.~Russo, and J.~Lobo.
\newblock Stable and {{Supported Semantics}} in {{Continuous Vector Spaces}}.
\newblock In \emph{Proceedings of the 18th {{International Conference}} on {{Principles}} of {{Knowledge Representation}} and {{Reasoning}}}, pages 59--68, 2020.
\newblock \doi{10.24963/kr.2020/7}.

\bibitem[Aspis et~al.(2022)Aspis, Broda, Lobo, and Russo]{aspisEmbed2SymScalableNeuroSymbolic2022}
Y.~Aspis, K.~Broda, J.~Lobo, and A.~Russo.
\newblock {{Embed2Sym}} - {{Scalable Neuro-Symbolic Reasoning}} via {{Clustered Embeddings}}.
\newblock In \emph{Proceedings of the 19th {{International Conference}} on {{Principles}} of {{Knowledge Representation}} and {{Reasoning}}}, 2022.

\bibitem[Badreddine et~al.(2022)Badreddine, {d'Avila Garcez}, Serafini, and Spranger]{badreddineLogicTensorNetworks2022}
S.~Badreddine, A.~{d'Avila Garcez}, L.~Serafini, and M.~Spranger.
\newblock Logic {{Tensor Networks}}.
\newblock \emph{Artificial Intelligence}, 303:\penalty0 103649, Feb. 2022.
\newblock ISSN 0004-3702.
\newblock \doi{10.1016/j.artint.2021.103649}.

\bibitem[Clark(1978)]{clarkNegationFailure1978}
K.~L. Clark.
\newblock Negation as {{Failure}}.
\newblock In \emph{Logic and {{Data Bases}}}, pages 293--322. {Springer US}, {Boston, MA}, 1978.
\newblock ISBN 978-1-4684-3384-5.
\newblock \doi{10.1007/978-1-4684-3384-5_11}.

\bibitem[Cunnington et~al.(2023)Cunnington, Law, Lobo, and Russo]{cunningtonNeuroSymbolicLearningAnswer2023}
D.~Cunnington, M.~Law, J.~Lobo, and A.~Russo.
\newblock Neuro-{{Symbolic Learning}} of {{Answer Set Programs}} from {{Raw Data}}.
\newblock In \emph{Proceedings of the {{Thirty-Second International Joint Conference}} on {{Artificial Intelligence}}}, pages 3586--3596. ijcai.org, 2023.
\newblock \doi{10.24963/IJCAI.2023/399}.

\bibitem[Dong et~al.(2019)Dong, Mao, Lin, Wang, Li, and Zhou]{dongNeuralLogicMachines2019}
H.~Dong, J.~Mao, T.~Lin, C.~Wang, L.~Li, and D.~Zhou.
\newblock Neural {{Logic Machines}}.
\newblock In \emph{International {{Conference}} on {{Learning Representations}} ({{ICLR}})}, page~22, 2019.

\bibitem[Eiter et~al.(2022)Eiter, Higuera, Oetsch, and Pritz]{eiterNeuroSymbolicASPPipeline2022a}
T.~Eiter, N.~Higuera, J.~Oetsch, and M.~Pritz.
\newblock A {{Neuro-Symbolic ASP Pipeline}} for {{Visual Question Answering}}.
\newblock \emph{Theory Pract. Log. Program.}, 22\penalty0 (5):\penalty0 739--754, 2022.
\newblock \doi{10.1017/S1471068422000229}.

\bibitem[Evans and Grefenstette(2018)]{evansLearningExplanatoryRules2018}
R.~Evans and E.~Grefenstette.
\newblock Learning {{Explanatory Rules}} from {{Noisy Data}}.
\newblock \emph{Journal of Artificial Intelligence Research}, 61:\penalty0 1--64, Jan. 2018.
\newblock ISSN 1076-9757.
\newblock \doi{10.1613/jair.5714}.

\bibitem[Evans et~al.(2021)Evans, Bosnjak, Buesing, Ellis, Reichert, Kohli, and Sergot]{evansMakingSenseRaw2021}
R.~Evans, M.~Bosnjak, L.~Buesing, K.~Ellis, D.~P. Reichert, P.~Kohli, and M.~J. Sergot.
\newblock Making sense of raw input.
\newblock \emph{Artificial Intelligence}, 299:\penalty0 103521, 2021.
\newblock \doi{10.1016/j.artint.2021.103521}.

\bibitem[Ferraris et~al.(2006)Ferraris, Lee, and Lifschitz]{ferrarisGeneralizationLinZhaoTheorem2006}
P.~Ferraris, J.~Lee, and V.~Lifschitz.
\newblock A generalization of the {{Lin-Zhao}} theorem.
\newblock \emph{Annals of Mathematics and Artificial Intelligence}, 47\penalty0 (1):\penalty0 79--101, June 2006.
\newblock ISSN 1573-7470.
\newblock \doi{10.1007/s10472-006-9025-2}.

\bibitem[Gao et~al.(2024)Gao, Inoue, Cao, and Wang]{gaoDifferentiableFirstorderRule2024}
K.~Gao, K.~Inoue, Y.~Cao, and H.~Wang.
\newblock A differentiable first-order rule learner for inductive logic programming.
\newblock \emph{Artificial Intelligence}, 331:\penalty0 104108, June 2024.
\newblock ISSN 0004-3702.
\newblock \doi{10.1016/j.artint.2024.104108}.

\bibitem[Gaunt et~al.(2017)Gaunt, Brockschmidt, Kushman, and Tarlow]{gauntDifferentiableProgramsNeural2017}
A.~L. Gaunt, M.~Brockschmidt, N.~Kushman, and D.~Tarlow.
\newblock Differentiable {{Programs}} with {{Neural Libraries}}.
\newblock In \emph{Proceedings of the 34th {{International Conference}} on {{Machine Learning}}}, pages 1213--1222, July 2017.

\bibitem[Gelfond and Lifschitz(1988)]{gelfondStableModelSemantics1988}
M.~Gelfond and V.~Lifschitz.
\newblock The stable model semantics for logic programming.
\newblock In \emph{{{ICLP}}/{{SLP}}}, volume~88, pages 1070--1080, 1988.

\bibitem[Hitzler and Sarker(2022)]{hitzlerNeuroSymbolicArtificialIntelligence2022}
P.~Hitzler and M.~K. Sarker, editors.
\newblock \emph{Neuro-{{Symbolic Artificial Intelligence}}: {{The State}} of the {{Art}}}, volume 342 of \emph{Frontiers in {{Artificial Intelligence}} and {{Applications}}}.
\newblock IOS Press, 2022.
\newblock ISBN 978-1-64368-244-0.
\newblock \doi{10.3233/FAIA342}.

\bibitem[Hitzler et~al.(2023)Hitzler, Sarker, and Eberhart]{hitzlerCompendiumNeurosymbolicArtificial2023}
P.~Hitzler, M.~K. Sarker, and A.~Eberhart, editors.
\newblock \emph{Compendium of {{Neurosymbolic Artificial Intelligence}}}, volume 369 of \emph{Frontiers in {{Artificial Intelligence}} and {{Applications}}}.
\newblock IOS Press, 2023.
\newblock ISBN 978-1-64368-406-2.
\newblock \doi{10.3233/FAIA369}.

\bibitem[Lecun et~al.(1998)Lecun, Bottou, Bengio, and Haffner]{lecunGradientbasedLearningApplied1998a}
Y.~Lecun, L.~Bottou, Y.~Bengio, and P.~Haffner.
\newblock Gradient-based learning applied to document recognition.
\newblock \emph{Proceedings of the IEEE}, 86\penalty0 (11):\penalty0 2278--2324, Nov. 1998.
\newblock ISSN 1558-2256.
\newblock \doi{10.1109/5.726791}.

\bibitem[Lin and Zhao(2004)]{linASSATComputingAnswer2004}
F.~Lin and Y.~Zhao.
\newblock {{ASSAT}}: Computing answer sets of a logic program by {{SAT}} solvers.
\newblock \emph{Artificial Intelligence}, 157\penalty0 (1):\penalty0 115--137, Aug. 2004.
\newblock ISSN 0004-3702.
\newblock \doi{10.1016/j.artint.2004.04.004}.

\bibitem[Manhaeve et~al.(2018)Manhaeve, Dumancic, Kimmig, Demeester, and De~Raedt]{manhaeveDeepProbLogNeuralProbabilistic2018}
R.~Manhaeve, S.~Dumancic, A.~Kimmig, T.~Demeester, and L.~De~Raedt.
\newblock {{DeepProbLog}}: {{Neural Probabilistic Logic Programming}}.
\newblock In \emph{Advances in {{Neural Information Processing Systems}} 31}, pages 3749--3759, 2018.

\bibitem[Mintz et~al.(2009)Mintz, Bills, Snow, and Jurafsky]{mintzDistantSupervisionRelation2009}
M.~Mintz, S.~Bills, R.~Snow, and D.~Jurafsky.
\newblock Distant supervision for relation extraction without labeled data.
\newblock In \emph{Proceedings of the {{Joint Conference}} of the 47th {{Annual Meeting}} of the {{ACL}} and the 4th {{International Joint Conference}} on {{Natural Language Processing}} of the {{AFNLP}}}, pages 1003--1011, Aug. 2009.

\bibitem[Riegel et~al.(2020)Riegel, Gray, Luus, Khan, Makondo, Akhalwaya, Qian, Fagin, Barahona, Sharma, Ikbal, Karanam, Neelam, Likhyani, and Srivastava]{riegelLogicalNeuralNetworks2020a}
R.~Riegel, A.~Gray, F.~Luus, N.~Khan, N.~Makondo, I.~Y. Akhalwaya, H.~Qian, R.~Fagin, F.~Barahona, U.~Sharma, S.~Ikbal, H.~Karanam, S.~Neelam, A.~Likhyani, and S.~Srivastava.
\newblock Logical {{Neural Networks}}.
\newblock \emph{arXiv preprint: arXiv:2006.13155}, June 2020.
\newblock \doi{10.48550/arXiv.2006.13155}.

\bibitem[Rockt{\"a}schel and Riedel(2017)]{rocktaschelEndtoendDifferentiableProving2017}
T.~Rockt{\"a}schel and S.~Riedel.
\newblock End-to-end differentiable proving.
\newblock In \emph{Advances in {{Neural Information Processing Systems}}}, pages 3788--3800, 2017.

\bibitem[Sakama et~al.(2021)Sakama, Inoue, and Sato]{sakamaLogicProgrammingTensor2021}
C.~Sakama, K.~Inoue, and T.~Sato.
\newblock Logic programming in tensor spaces.
\newblock \emph{Annals of Mathematics and Artificial Intelligence}, 89:\penalty0 1133--1153, Aug. 2021.
\newblock ISSN 1573-7470.
\newblock \doi{10.1007/s10472-021-09767-x}.

\bibitem[Sato et~al.(2020)Sato, Sakama, and Inoue]{sato3valuedSemanticsSupported2020}
T.~Sato, C.~Sakama, and K.~Inoue.
\newblock From 3-valued {{Semantics}} to {{Supported Model Computation}} for {{Logic Programs}} in {{Vector Spaces}}.
\newblock In \emph{12th {{International Conference}} on {{Agents}} and {{Artificial Intelligence}}}, pages 758--765, Sept. 2020.
\newblock ISBN 978-989-758-395-7.

\bibitem[Sato et~al.(2023)Sato, Takemura, and Inoue]{satoEndtoendASPComputation2023}
T.~Sato, A.~Takemura, and K.~Inoue.
\newblock Towards end-to-end {{ASP}} computation.
\newblock \emph{arXiv preprint: arXiv:2306.06821}, June 2023.

\bibitem[Shindo et~al.(2023)Shindo, Pfanschilling, Dhami, and Kersting]{shindoAlphaILPThinking2023}
H.~Shindo, V.~Pfanschilling, D.~S. Dhami, and K.~Kersting.
\newblock \(\alpha\){{ILP}}: Thinking visual scenes as differentiable logic programs.
\newblock \emph{Mach. Learn.}, 112\penalty0 (5):\penalty0 1465--1497, 2023.
\newblock \doi{10.1007/S10994-023-06320-1}.

\bibitem[Takemura and Inoue(2022)]{takemuraGradientBasedSupportedModel2022}
A.~Takemura and K.~Inoue.
\newblock Gradient-{{Based Supported Model Computation}} in~{{Vector Spaces}}.
\newblock In \emph{Logic {{Programming}} and {{Nonmonotonic Reasoning}}}, pages 336--349, 2022.
\newblock ISBN 978-3-031-15707-3.
\newblock \doi{10.1007/978-3-031-15707-3_26}.

\bibitem[Takemura and Inoue(2024)]{takemuraDifferentiableLogicProgramming2024}
A.~Takemura and K.~Inoue.
\newblock Differentiable {{Logic Programming}} for {{Distant Supervision}}.
\newblock \emph{arXiv preprint: arXiv:2408.12591}, Aug. 2024.
\newblock \doi{10.48550/arXiv.2408.12591}.
\newblock Full version of this paper.

\bibitem[Thoma(2017)]{thomaHASYv2Dataset2017b}
M.~Thoma.
\newblock The {{HASYv2}} dataset.
\newblock \emph{CoRR}, abs/1701.08380, 2017.

\bibitem[Tsamoura et~al.(2021)Tsamoura, Hospedales, and Michael]{tsamouraNeuralSymbolicIntegrationCompositional2021}
E.~Tsamoura, T.~Hospedales, and L.~Michael.
\newblock Neural-{{Symbolic Integration}}: {{A Compositional Perspective}}.
\newblock \emph{Proceedings of the AAAI Conference on Artificial Intelligence}, 35\penalty0 (6):\penalty0 5051--5060, May 2021.
\newblock ISSN 2374-3468.

\bibitem[Winters et~al.(2022)Winters, Marra, Manhaeve, and Raedt]{wintersDeepStochLogNeuralStochastic2022a}
T.~Winters, G.~Marra, R.~Manhaeve, and L.~D. Raedt.
\newblock {{DeepStochLog}}: {{Neural Stochastic Logic Programming}}.
\newblock In \emph{Thirty-{{Sixth AAAI Conference}} on {{Artificial Intelligence}}, {{AAAI}} 2022}, pages 10090--10100. AAAI Press, 2022.
\newblock \doi{10.1609/AAAI.V36I9.21248}.

\bibitem[Xu et~al.(2018)Xu, Zhang, Friedman, Liang, and den Broeck]{DBLP:conf/icml/XuZFLB18}
J.~Xu, Z.~Zhang, T.~Friedman, Y.~Liang, and G.~V. den Broeck.
\newblock A semantic loss function for deep learning with symbolic knowledge.
\newblock In \emph{Proceedings of the 35th International Conference on Machine Learning, {{ICML}} 2018}, volume~80, pages 5498--5507, 2018.

\bibitem[Yang et~al.(2020)Yang, Ishay, and Lee]{yangNeurASPEmbracingNeural2020}
Z.~Yang, A.~Ishay, and J.~Lee.
\newblock {{NeurASP}}: {{Embracing Neural Networks}} into {{Answer Set Programming}}.
\newblock In \emph{Proceedings of the {{Twenty-Ninth International Joint Conference}} on {{Artificial Intelligence}}}, pages 1755--1762, July 2020.
\newblock ISBN 978-0-9992411-6-5.
\newblock \doi{10.24963/ijcai.2020/243}.

\end{thebibliography}


\newpage
\appendix

\section{Experimental Details}

\subsection{Datasets}

For tasks involving MNIST digits, specifically MNIST Addition 1 and 2, Member 3 and 5, and Add2x2, we used the MNIST dataset from \citet{lecunGradientbasedLearningApplied1998a}.
The math operator images in Apply2x2 task comes from the HASY dataset by \citet{thomaHASYv2Dataset2017b}, and we only used the addition, subtraction and multiplication operators.
We then generated random combinations of images for each task and used them as training data.

\subsection{Computational Environment}

The experimental environment is a desktop computer with an AMD Ryzen 7950X (16c/32t), 128GB RAM, and an NVIDIA A4000 (16GB), with settings to utilize the GPU as much as possible.

\subsection{Neural Network}

\subsubsection{Neural Network Architectures}

\begin{table}[ht]
    \centering
    \caption{Neural network architectures used in the experiments}
    \begin{tabular}{p{0.15\linewidth}p{0.34\linewidth}p{0.34\linewidth}}
        \hline
        Task       & Network & Neural Baseline \\
        \hline
        Addition 1 & Conv2D(6,5), MaxPool2D(2,2), ReLU, Conv2D(16,5), MaxPool2D(2,2), ReLU, Linear(120), ReLU, Linear(84), ReLu, Linear(10), Softmax & Identical to the network on the left, except the last two layers Linear(19), LogSoftmax \\
        Addition 2 & Identical to Addition 1 & Identical to Addition 1, except the last two layers Linear(199), LogSoftmax \\
        Add2x2     & Identical to Addition 1 & Identical to Addition 1, except the last two layers Linear(19), LogSoftmax \\
        Apply2x2   & Identical to Addition 1, except the penultimate layer Linear(3) & Identical to Addition 1, except the last two layers Linear(1200), LogSoftmax \\
        Member 3   & Identical to Addition 1 & Identical to Addition 1, except the last two layers Linear(1), Sigmoid \\
        Member 5   & Identical to Addition 1 & Identical to Addition 1, except the last two layers Linear(1), Sigmoid \\
        \hline
    \end{tabular}
    \label{tab:neuralarchitectures}
\end{table}

Table \ref{tab:neuralarchitectures} summarizes the neural network architectures used in the experiments.
Conv2D(o,k) is a 2-dimensional convolution layer with o output channels and a kernel size of k.
MaxPool2D(k,s) is a 2-dimentional MaxPool layer with a kernel size of k, and a stride of s.
Lin(n) is a fully connected layer of output size n.

\subsubsection{Neural Baselines}

We report the results of the neural baselines in Table \ref{tab:neuralbaseline}.
The experimental settings are identical to the ones presented in the main text, i.e., the same number of training data, optimizer settings (Adam with learning rate of 0.001), number of epochs (1), batch size (1), time out (30 mins), and number trials (5 per experiment).
The architectures used in this experiment are shown in the ``Neural Baseline`` column of Table \ref{tab:neuralarchitectures}.

Note that the accuracy figures in this table are \textit{not} directly comparable to the results in the main text.
The results reported in the main text are the accuracy figures for the digit and operator classifications, whereas the ones reported here are the accuracy for the labels themselves, i.e., the sums in addition tasks or the binary membership labels in membership tasks.

\begin{table}[ht]
    \centering
    \caption{Neural baseline results}
    \begin{tabular}{l c c}
        \hline
        Task       &  Accuracy (\%)  & Time (s) \\
        \hline
        Addition 1 &  89.3 & 22.3 \\
        Addition 2 &   1.0 & 12.6 \\
        Add2x2     &  66.7 & 19.2 \\
        Apply2x2   &   0.1 & 24.3 \\
        Member 3   &  26.9 &  8.7 \\
        Member 5   &  40.6 &  9.5 \\
        \hline
    \end{tabular}
    \label{tab:neuralbaseline}
\end{table}

\section{Proof of Proposition 1}

\begin{proposition}(Embedding Models of Normal Logic Programs) \label{prop:appdx_embed_model}
    Let \(\matr{P} = (\matr{Q}, \matr{D})\) be an embedding of a ground normal logic program \(P\), \(dist(\matr{\cdot},\matr{\cdot})\) be a distance function in a metric space, \(\matr{v}\) be an interpretation vector representing \(I \subseteq B_P\), and \(\matr{w}\) be its complementary interpretation vector.
    Then, for a suitable interpretation vector \(\matr{v}\),
    \begin{equation*}
        I \models comp(P) \ \mathrm{iff}\ dist\bigg(\matr{v}, \mathrm{min_1}\bigg(\mathbf{D}\Big(\mathbf{1}-\mathrm{min_1}\big(\mathbf{Q}(\mathbf{1}-\mathbf{w})\big)\Big)\bigg)\bigg) = 0
    \end{equation*}
\end{proposition}  
\begin{proof}
    Firstly, we show that \(\matr{1}-\mathrm{min_1}(\mathbf{Q}(\mathbf{1}-\mathbf{w}))\) corresponds to the evaluation of conjunctions in the rule bodies.
    We only show the case for a single row slice of the program matrix, but by construction it can be trivially applied to the entire matrix.
    By construction, a slice of the program matrix \(\matr{Q}_{i:}\) corresponds to the body of a rule \(R_i\).
    Then, the matrix-vector product \(\matr{Q}_{i:}\matr{w}\) computes the number of true literals in \(I\).
    When all literals in the body of \(R_i\) are true, the body is true in \(I\), and \(\sum_{j}|\matr{Q}_{ij}|=\matr{Q}_{i:}\matr{w}\).
    The conjunctions can also be computed as the negation of disjunctions of negated literals, using De Morgan's law, i.e., \(L_1 \wedge \dots \wedge L_n = \neg (\neg L_1 \vee \dots \vee \neg L_m)\).
    Thus, we count the number of false literals in the body of \(R_i\), by \(k=\matr{Q}_{i:}(\matr{1}-\matr{w})\).
    If \(k \geq 1\), then there is at least one literal that is false in \(I\), and if \(k = 0\) then there is no literal that is false in \(I\).
    The existence of at least one false literal is computed by \(\mathrm{min_1}(\mathbf{Q}_{i:}(\mathbf{1}-\mathbf{w}))\), which is 1 if there is a false literal and 0 otherwise.
    This corresponds to the evaluation of disjunctions of negated literals.
    Now we obtain the complement by subtracting from a 1-vector, i.e., \(\matr{1}-\mathrm{min_1}(\mathbf{Q}_{i:}(\mathbf{1}-\mathbf{w}))\), which corresponds to the negation of the disjunctions.
    Therefore, \(\matr{1}-\mathrm{min_1}(\mathbf{Q}_{i:}(\mathbf{1}-\mathbf{w}))\) corresponds to the evaluation of \(I \models R_i\) in vector spaces.

    Secondly, we show that \(\mathrm{min_1}\bigg(\mathbf{D}\Big(\mathbf{1}-\mathrm{min_1}\big(\mathbf{Q}(\mathbf{1}-\mathbf{w})\big)\Big)\bigg)\) corresponds to the evaluation of disjunctions of the rule bodies that share the same head atoms.
    Let \(\{R^{a_i}\} = \{a_i \leftarrow B_j,\dots,a_i \leftarrow B_t\}\) be the ordered set of rules that share the same head atom \(a_i\), where \(B_j\) denote the rule bodies, and \(B_j \vee \dots \vee B_t\) be the disjunction of the rule bodies.
    We assume the starting index \(j\) is appropriately set such that \(j\) corresponds to the first row of submatrix corresponding to \(\{R^{a_i}\}\) in the program matrix \(\matr{Q}\).
    Using the aforementioned result, we can count the number of true rule bodies as: \(d_i = \sum_{j}^{t}(\matr{1}-\mathrm{min_1}(\mathbf{Q}_{j:}(\mathbf{1}-\mathbf{w})))\).
    By construction of the head matrix \(\matr{D}\), we can replace the summation by matrix multiplication: \(d_i = \matr{D}_{i:}(\matr{1}-\mathrm{min_1}(\mathbf{Q}(\mathbf{1}-\mathbf{w})))\).
    For an atom \(a_i\), \(d_i \geq 1\) if there is at least one rule body that is true in \(I\), and \(d_i = 0\) otherwise.
    Thus, \(\matr{h}_i = \mathrm{min_1}\bigg(\mathbf{D}_{i:}\Big(\mathbf{1}-\mathrm{min_1}\big(\mathbf{Q}(\mathbf{1}-\mathbf{w})\big)\Big)\bigg) = 1\) if there is at least one rule body that is true in \(I\) and in the disjunction \(B_j \vee \dots \vee B_t\), and 0 otherwise.
    Computing \(\matr{h}_i\) also corresponds to the evaluation of \(I \models comp(R_{a_i})\).
    Let us introduce a column vector \(\matr{h}^I = \mathrm{min_1}\bigg(\mathbf{D}\Big(\mathbf{1}-\mathrm{min_1}\big(\mathbf{Q}(\mathbf{1}-\mathbf{w})\big)\Big)\bigg)\).
    Then the elements \(\matr{h}^I_i\) are 1 if there is at least one rule body for \(a_i\) that is true in \(I\), and 0 otherwise.
    Therefore, \(\mathrm{min_1}\bigg(\mathbf{D}\Big(\mathbf{1}-\mathrm{min_1}\big(\mathbf{Q}(\mathbf{1}-\mathbf{w})\big)\Big)\bigg)\) corresponds to the evaluation of disjunctions of the rule bodies that share the same head atoms.

    Now we are ready to prove the if-and-only-if relation in the proposition.
    Let \(\matr{h}^I = \mathrm{min_1}\bigg(\mathbf{D}\Big(\mathbf{1}-\mathrm{min_1}\big(\mathbf{Q}(\mathbf{1}-\mathbf{w})\big)\Big)\bigg)\), then the second part is simplified to \(dist(\matr{v}, \matr{h}^I)) = 0\).
    \begin{itemize}
        \item If \(I \models comp(P)\), then \(dist(\matr{v}, \matr{h}^I) = 0\). \\
        Suppose \(I \models comp(R_{a_i})\), then there is at least one rule body that is true in \(I\), so \(\matr{h}_i^{I} = 1\).
        Otherwise, when we have \(I \not\models comp(R_{a_i})\), \(\matr{h}_i^{I} = 0\).
        Therefore, it holds that \(\matr{h}_i^{I} = \matr{v}_i\), and since this can be trivially applied to the entire vector using the index \(i\), we have \(\matr{v} = \matr{h}_i^{I}\), i.e., \(dist(\matr{v}, \matr{h}^I) = 0\).
        
        \item If \(dist(\matr{v}, \matr{h}^I) = 0\), then \(I \models comp(P)\). \\
        Consider \(dist(\matr{v}, \matr{h}^I) = 0\). 
        For \(\matr{h}^I_i = 1\), there is at least one rule body that is true in \(I\), and for \(\matr{h}^I_i = 0\), there is no rule body that is true in \(I\). 
        Since we have \(\matr{v}_i = \matr{h}^I_i\), for \(\matr{v}_i = \matr{h}^I_i = 1\), \(a_i \leftrightarrow \bigvee_{R_j \in heads(P, a_i)} body(R_j)\) is satisfied and denote \(I \models comp(R_{a_i})\), and for \(\matr{v}_i = \matr{h}^I_i = 0\) denote \(I \not\models comp(R_{a_i})\).
        Since the index \(i\) is arbitrary, we conclude \(I \models comp(P)\).
    \end{itemize}
\end{proof}

\section{Loop Formula and Stable Model Semantics}

The content of the following section is closely related to a study by Sato et al. \cite{satoEndtoendASPComputation2023}, which addresses the computation of stable models in vector spaces.

\subsection{Embedding Loop Formula}

The Lin-Zhao theorem \cite{linASSATComputingAnswer2004} shows how to turn a normal logic program into a set of propositional formulas that describe the stable models of the program.
A set \(\emptyset \subset L \subseteq B_P\) is a \textit{loop} of a logic program if it induces a non-trivial strongly connected subgraph of the positive dependency graph of \(P\).
A subgraph is said to be non-trivial if each pair of atoms in \(L\) is connected by at least one path of non-zero length.
The set of all loops of \(P\) is denoted by \(loop(P)\).
The number of loops may be exponential in \(|B_P|\), and the program \(P\) is \textit{tight} if \(loop(P)=\emptyset\).
For \(L \subseteq B_P\), the \textit{external support} is given by \(ES(L)=\{ r \in P | head(r), body^+(r) \cap L = \emptyset \}\).
The \textit{Loop formula} can be constructed to exclude loops without external support.
The \textit{disjunctive loop formula} is \(LF(L)^{\vee} = (\bigvee_{a\in L} a) \rightarrow (\bigvee_{B\in EB_P(L)}BF(B))\).
According to \citet{ferrarisGeneralizationLinZhaoTheorem2006}, it is equivalent to the \textit{conjunctive loop formula} \(LF(L)^{\wedge} = (\bigwedge_{a\in L} a) \rightarrow (\bigvee_{B\in EB_P(L)}BF(B))\).
The loop formula enforces all atoms in \(L\) to be false whenever \(L\) is not externally supported.

We introduce two types of matrices; a \textit{loop matrix} which corresponds to the enumeration of loops \(L\), and a \textit{external support matrix} which is an indicator matrix for the external support.
\begin{definition} (Loop matrix)
    Let the \(o\)-th loop be denoted by \(L_o \in L \, (1 \leq o \leq O)\) where \(O\) is the total number of loops in the program.
    Then, \(L\) is embedded into the matrix \(\matr{L}\) by stacking a binary row vector \(\matr{L}_{o,:}\), where \(\matr{L}_{o,i}=1\) if \(\mathrm{idx}(i) \in L_o\) and \(\matr{L}_{o,i}=0\) otherwise.
\end{definition}

\begin{definition} (External Support matrix)
    Let \(ES(L_o)\) be the external support for the loop \(L_o\).
    Then, define a binary matrix \(\matr{ES}^{L_o} \in \{0,1\}^{N \times R}\) where \(\matr{ES}^{L_o}_{i,j}=1\) if the \(j\)-th rule is the external support for the atom \(a_i \in L\), and \(\matr{ES}^{L_o}_{i,j}=0\) otherwise.
\end{definition}
Since \(ES^{L_o}\) is defined for each loop, there may be an exponential number of external support matrices.

\begin{example}
    Consider \(P_1\), which has a single loop \(L = \{\{a\}\}\). Then according to the previous definitions construct \(\matr{L}^{P_1} \in \{0,1\}^{1\times3}\), and \(ES^{L_1} \in \{0,1\}^{3\times3}\) as follows:
    \begin{equation}
        \let\quad\enspace
        \matr{L}^{P_1}  = \bordermatrix{ & a & b & c \cr
                                     L_1 & 1 & 0 & 0    
        } \;
        \matr{ES}^{L_1}  = \bordermatrix { & R_1 & R_2 & R_3 \cr
                                    a & 1   & 0   & 0   \cr
                                    b & 0   & 0   & 0   \cr
                                    c & 0   & 0   & 0   
        }
    \end{equation}
\end{example}

\subsection{Evaluating Embedded Loop Formula}
We shall be using the conjunctive loop formula definition.
Recall the definition of the conjunctive loop formula: \(LF(L)^{\wedge} = (\bigwedge_{a\in L} a) \rightarrow (\bigvee_{B\in EB_P(L)}BF(B))\).
This can then be translated into the following; \(\neg (\bigwedge_{a\in L} a) \vee (\bigvee_{B\in EB_P(L)}BF(B))\).
The first term can be trivially evaluated by a simple matrix-vector multiplication; given an interpretation vector \(\matr{v}\) representing an interpretation \(I\), \(\matr{L}_{o}(1-\matr{v})\) computes the number of loop atoms that are \textit{false} in \(I\).

The second term requires that there must be at least one external support rule body that is satisfied in \(I\).
We start by evaluating the rule bodies in a similar manner to the evaluation of implication rules as in Proposition \ref{prop:appdx_embed_model}.
Given an embedding \(\matr{Q}\) of a program and a complementary interpretation vector \(\matr{w}\), we obtain a indicator vector \(\matr{z} \in \{0,1\}^{R \times 1}\) where \(\matr{z}_i=1\) indicates the rule body is satisfied: \(\matr{z} = \mathbf{1}-\mathrm{min_1}(\mathbf{Q}(\mathbf{1}-\mathbf{w}))\). 
Then, we combine \(\matr{z}\) with the loop matrix and the external support matrix: \(\matr{L}_o \matr{ES}^{L_o} \matr{z}\).
Intuitively, \(\matr{ES}^{L_o} \matr{z}\) is equivalent to checking whether at least one of the external support bodies holds for each \(a_i \in L_o\), and multiplication with \(\matr{L}_o\) amounts to counting the number of true external support bodies.

For a loop \(L_o \in L\), we then combine the two terms in the disjunction:
\begin{equation}
    u_o = \mathrm{min}_1 \big( \matr{L}_o(1-\matr{v}) + \matr{L}_o \matr{ES}^{L_o} \matr{z} \big)
\end{equation}
Note that \(u_o\) is a scalar and \(u_o=1\) indicates the loop formula for \(L_o\) is satisfied.
We then apply this to all loops in the program, and count the number of unsatisfied loop formula:
\begin{equation}
    \mathcal{L}_{loop} = \sum_{o=1}^{O} (1 - u_o)
\end{equation}
This term will be 0 iff the loop formula is satisfied.

\subsection{Stable Model Semantics in Vector Spaces}

The Lin-Zhao theorem essentially states that \(I\) is a stable model of a program iff \(I \models comp(P) \cup LF\) where \(comp(P)\) is the completion and \(LF\) is the loop formula.
Thus, combining Proposition \ref{prop:embed_model} and the aforementioned embedded loop formula evaluation, we can construct the following loss function which is 0 iff when \(\matr{v}\) (and corresponding \(I\)) is a stable model of \(P\).\\\
\begin{align}
    \mathcal{L}_{SM} &= \mathcal{L}_{comp.} + \mathcal{L}_{loop} \nonumber \\ 
    &= dist\bigg(\matr{v}, \mathrm{min_1}\bigg(\mathbf{D}\Big(\mathbf{1}-\mathrm{min_1}\big(\mathbf{Q}(\mathbf{1}-\mathbf{w})\big)\Big)\bigg)\bigg) \nonumber \\
    &+ \sum_{o=1}^{O} \bigg(1 - \mathrm{min}_1 \big( \matr{L}_o(1-\matr{v}) + \matr{L}_o \matr{ES}^{L_o} \matr{z} \big)\bigg)
\end{align}
While this function is almost everywhere differentiable and we can indeed use gradient-based methods to search for stable models, it is still an open question whether this continuous method is more efficient than the discrete method commonly used in existing solvers.

\end{document}